\documentclass[11pt, a4paper]{article}

\usepackage[a4paper]{geometry}
\usepackage[english]{babel}
\usepackage{bm}

\usepackage{diagbox}

\usepackage[utf8]{inputenc}
\usepackage{xspace}
\usepackage{amsmath,amsthm,amssymb,mathtools}
\usepackage{url}
\usepackage{dsfont,bbm}
\usepackage{natbib}
\usepackage{algorithm}
\usepackage[noend]{algpseudocode}
\usepackage{xcolor}
\usepackage{tikz}
\usepackage{graphicx}
\usepackage{subcaption}
\usepackage{multirow}

\newcounter{phase}[algorithm]
\newlength{\phaserulewidth}
\newcommand{\setphaserulewidth}{\setlength{\phaserulewidth}}
\newcommand{\phase}[1]{
	\vspace{-1.25ex}
	\Statex\leavevmode\llap{\rule{\dimexpr\labelwidth+\labelsep}{\phaserulewidth}}\rule{\linewidth}{\phaserulewidth}
	\Statex\strut\refstepcounter{phase}\textit{State~\thephase~--~#1}
	\vspace{-1.25ex}\Statex\leavevmode\llap{\rule{\dimexpr\labelwidth+\labelsep}{\phaserulewidth}}\rule{\linewidth}{\phaserulewidth}}
\makeatother

\setphaserulewidth{.5pt}

\renewcommand{\labelenumi}{(\alph{enumi})}
\renewcommand\theenumi\labelenumi
\newcommand{\filtt}{\mathcal{F}_t}

\newtheorem{lemma}{Lemma}
\newtheorem{theorem}{Theorem}

\newtheorem{corollary}{Corollary}

\renewenvironment{proof}
{\begin{trivlist}\item\textbf{Proof.}}
{\hspace*{\fill}$\Box$\end{trivlist}}

\newenvironment{proofwithoutbox}
{\begin{trivlist}\item\textbf{Proof.}}
{\end{trivlist}}

\newcommand{\ooea}{(1+1)~EA\xspace}
\newcommand{\sdooea}{SD-(1+1)~EA\xspace}
\newcommand{\oneoneea}{\ooea}

\newcommand{\oplea}{(1+$\lambda$)~EA\xspace}
\newcommand{\onelambdaea}{\oplea}
\newcommand{\saonelambdaea}{SASD-(1+$\lambda$)~EA\xspace}

\newcommand{\oofea}{(1+1)~FEA$_\beta$\xspace}
\newcommand{\om}{\textsc{OneMax}\xspace}
\newcommand{\onemax}{\om}
\newcommand{\jump}{\textsc{Jump}\xspace}
\newcommand{\lo}{\textsc{Lead\-ing\-Ones}\xspace}

\newcommand{\leadingones}{\lo}
\newcommand{\trap}{\textsc{Trap}\xspace}
\newcommand{\R}{\ensuremath{\mathbb{R}}}
\newcommand{\N}{\ensuremath{\mathbb{N}}}

\newcommand{\indic}[1]{\mathds{1}_{#1}}
\newcommand{\card}[1]{\lvert #1\rvert}
\newcommand{\ones}[1]{\lvert #1\rvert_1}

\newcommand{\pre}{\textsc{pre}}
\newcommand{\suff}{\textsc{suff}}
\newcommand{\needhighmut}{\textsc{NeedHighMut}\xspace}
\DeclareMathOperator{\Prob}{Pr}

\newcommand{\prob}[1]{\Prob\left(#1\right)}
\newcommand{\expect}[1]{\mathrm{E}\left(#1\right)}

\DeclareMathOperator{\gap}{gap}
\DeclareMathOperator{\im}{Im}

\newcommand{\ie}{i.\,e.\xspace}
\newcommand{\wlo}{w.\,l.\,o.\,g.\xspace}
\newcommand{\eg}{e.\,g.\xspace}

\title{Self-Adjusting Evolutionary Algorithms for Multimodal Optimization}

\author{
  Amirhossein Rajabi\\
  Technical University of Denmark \\
	Kgs. Lyngby \\
	Denmark \\
amraj@dtu.dk \\
  \and
    Carsten Witt\\
  Technical University of Denmark \\
	Kgs. Lyngby \\
	Denmark \\
cawi@dtu.dk \\
}

\begin{document}

\maketitle

\begin{abstract}
Recent theoretical research has shown that self-adjusting and self-adaptive mechanisms 
can provably outperform static settings in evolutionary algorithms for binary search spaces. However, 
the vast majority of these studies focuses on unimodal functions which do not require the algorithm 
to flip several bits simultaneously to make progress. In fact, existing self-adjusting algorithms are not designed 
to detect local optima and do not have any obvious benefit to cross large Hamming gaps.

We suggest a mechanism called stagnation detection that can be added as a module to existing evolutionary algorithms (both 
with and without prior self-adjusting schemes). Added to a simple (1+1)~EA, we prove an expected runtime on the 
well-known \emph{Jump} benchmark that corresponds to an asymptotically optimal parameter setting and outperforms 
other mechanisms for multimodal optimization like heavy-tailed mutation. We also investigate the module in the context 
of a self-adjusting (1+$\lambda$)~EA and show that it combines the previous benefits of this algorithm on unimodal problems 
with more efficient multimodal optimization. 

To explore the limitations of the approach, we additionally present an example 
where both self-adjusting mechanisms, including stagnation detection, do not help to find a beneficial 
setting of the mutation rate. 
Finally, we investigate our module for stagnation detection experimentally. 
\end{abstract}

\section{Introduction}
Recent theoretical research on self-adjusting algorithms in discrete search spaces has produced a remarkable 
body of results showing that self-adjusting and self-adaptive mechanisms outperform  static parameter settings.
Examples include an analysis of the well-known (1+$(\lambda,\lambda)$)~GA using a $1/5$-rule to adjust its mutation rate
 on \om \citep{DoerrDoerrAlgorithmica18}, 
of a self-adjusting (1+$\lambda$)~EA sampling offspring with different mutation rates \citep{DoerrGWYAlgorithmica19}, matching 
the parallel black-box complexity of the \om function, and 
a self-adaptive variant of the latter \citep{DoerrWY18}. Furthermore,  
self-adjusting schemes for 
algorithms over the search space $\{0,\dots,r\}^n$ for $r>1$ provably outperform static settings \citep{DoerrDoerrKoetzingAlgorithmica18} 
of the mutation operator. Self-adjusting schemes  
are also closely related 
to hyper-heuristics which, \eg, can dynamically choose between 
different mutation 
operators and therefore outperform static settings \citep{LOWECJ19}. Besides the mutation probability, other parameters like the population sizes 
may be adjusted during the run of an evolutionary algorithm (EA) and analyzed from a runtime perspective \citep{LassigSudholtFOGA11}. Moreover, 
there is much empirical evidence (\eg 
\citep{DoerrEtAlTheoryGuidedBenchmark18,DoerrWagnerPPSN18,RodionovaABDGECCO19,FajardoGECCO19}) showing that 
parameters of EAs should be adjusted during its run to optimize 
its runtime. See also the survey article \citep{DoerrDoerrParameterBookChapter} for an in-depth coverage of parameter 
control, self-adjusting algorithms, and theoretical runtime results.

A common feature of existing self-adjusting schemes is that they use different settings of a parameter (\eg, the 
mutation rate) and -- in some way -- measure and compare the progress achievable with the different settings. For 
example, the 2-rate \onelambdaea from  \cite{DoerrGWYAlgorithmica19} samples $\lambda/2$ of the offspring 
with strength~$r/2$ (where we define strength as the expected number of flipping bits, \ie, $n$ times the mutation probability) and the other half with strength $2r$. The strength is afterwards adjusted to the 
one used by a fittest offspring. Similarly, the $1/5$-rule \citep{DoerrDoerrAlgorithmica18} increases the mutation rate if fitness improvements 
happen frequently and decreases it otherwise. This requires that the algorithm is likely enough to make \emph{some} improvements 
with the different parameters tried or, at least, 
that the smallest disimprovement observed in unsuccessful mutations 
gives reliable hints on the choice of the parameter. However, there are situations where 
the algorithm cannot make progress and does not learn from unsuccessful mutations either. This can be the case when the 
algorithm reaches local optima escaping from which requires an unlikely event (such as flipping many bits simultaneously) to happen. Classical 
self-adjusting algorithms would observe many unsuccessful steps in such situations and suggest to set the mutation rate to its minimum although 
that might not be the best choice to leave the local optimum. In fact, the vast majority of runtime results for self-adjusting EAs is concerned 
with unimodal functions that have no other local optima than the global optimum. An exception is the work 
\citep{DangL16} which considers a self-adaptive EA allowing two different mutation probabilities 
on a specifically designed multimodal problem. Altogether, there is a lack of theoretical results 
giving guidance on how to design self-adjusting algorithms that can leave local optima efficiently.

In this paper, we address this question and propose a self-adjusting mechanism called \emph{stagnation detection} 
that adjusts mutation rates when the algorithm has reached a local optimum. In contrast to previous self-adjusting 
algorithms this mechanism is likely to increase the mutation in such situations, leading to a more efficient escape from 
local optima. This idea has been mentioned before, \eg, in the context of population sizing in stagnation \citep{EibenMVPPSN04}; also, 
recent empirical studies of the above-mentioned 2-rate \onelambdaea,
 handling of stagnation by increasing the variance was explicitly suggested in \cite{YeDoerrBaeckCEC19}.
 Our contribution has several advantages over previous discussion of stagnation detection: 
it represents a simple module that can be added to several existing evolutionary algorithms 
with little effort, it provably does not change the behavior of the algorithm on 
unimodal functions (except for small error terms), allowing the transfer of previous results, 
and we provide rigorous runtime analyses showing general upper bounds for multimodal functions including 
its benefits on the well-known \jump 
benchmark function.

In a nutshell, our stagnation detection mechanism works in the setting of pseudo-boolean optimization 
and standard bit mutation. Starting from strength $r=1$, it increases the  strength from $r$ to~$r+1$ after a long 
waiting time without improvement has elapsed, meaning it is unlikely that 
an improving bit string at Hamming distance~$r$ exists. 
This approach bears some resemblance with variable neighborhood search (VNS) \citep{HansenMladenovic18}; however, 
the idea of VNS is to apply local search with a fixed neighborhood until reaching a local 
optimum  and then to adapt the neighborhood structure. There have also been 
so-called quasirandom evolutionary algorithms \citep{DoerrFouzWittGECCO10} that search 
the set of Hamming neighbors of 
a search point more 
systematically; however, these approaches do not change the expected number 
of bits flipped. In contrast, 
our stagnation detection uses the whole time an unbiased  
randomized global search operator 
in an EA and just adjusts the underlying mutation probability. Statistical significance of 
long waiting times is used, indicating that improvements at Hamming distance~$r$ are unlikely to exist;  
this is rather remotely related to (but clearly inspired by)  the estimation-of-distribution algorithm  
sig-cGA \cite{DoerrKrejcaGECCO18} that uses statistical significance to counteract genetic drift. 

This paper is structured as follows: In Section~\ref{sec:preliminaries}, we introduce the concrete mechanism 
for stagnation detection and employ it in the context of a simple, static \oneoneea and the already self-adjusting 
2-rate \onelambdaea. Moreover, we collect tools for the analysis that are used in the rest of the paper. 
Section~\ref{sec:upper} deals with concrete runtime bounds for the \oneoneea and \onelambdaea 
with stagnation detection. Besides 
general upper bounds, we prove a concrete result for the \jump benchmark function that is asymptotically optimal 
for algorithms using standard bit mutation and outperforms previous mutation-based algorithms for this function 
like the heavy-tailed EA from  \cite{DoerrLMNGECCO17}. Elementary techniques are sufficient to show these results. To explore the limitations of stagnation detection and 
other self-adjusting schemes, we propose in Section~\ref{sec:needhighmut} a function where these mechanisms 
provably fail to set the mutation rate to a beneficial regime. 
As a technical tool, we use drift analysis and analyses of occupation times for 
processes with strong drift. To that purpose, we use a theorem by Hajek \citep{HajekDrift} 
on occupation times that, to the best of the knowledge, was not used for the analysis 
of randomized search heuristics before and may be of independent interest. Finally, in Section~\ref{sec:experiments}, we add some empirical results, 
showing that the asymptotically smaller runtime of our algorithm on \jump is also visible for small problem dimensions. 
We finish with some conclusions. 

\section{Preliminaries}
\label{sec:preliminaries}
We shall now formally define the algorithms analyzed and present some fundamental 
tools for the analysis.

\subsection{Algorithms}
We are concerned with pseudo-boolean functions $f\colon\{0,1\}^n\to \R$ that 
\wlo{} are to be maximized. A simple and well-studied EA studied in many runtime analyses (\eg, \cite{DrosteJW02}) 
is the \oneoneea displayed in Algorithm~\ref{alg:oneone-classic}. It uses a standard bit mutation with strength~$r$, where 
$1\le r\le n/2$, which means that every bit is flipped independently with probability~$r/n$. Usually, 
$r=1$ is used, which is the optimal strength on linear functions \citep{WittCPC2013}. Smaller strengths lead 
to less than $1$ bit being flipped in expectation, and strengths above $n/2$ in binary search spaces are considered ``ill-natured'' 
\citep{AntipovDKFOGA19} since a mutation at a bit should not be more likely than a non-mutation.

\begin{algorithm}
	\caption{\ooea with static strength~$r$}
	\label{alg:oneone-classic}
	\begin{algorithmic}
		\State Select $x$ uniformly at random from $\{0, 1\}^n$
		\For{$t \gets 1, 2, \dots$}
		\State Create $y$ by flipping each bit in a copy of $x$ independently with probability~$\frac{r}{n}$.
		\If{$f(y) \ge f(x)$}
		\State $x \gets y$.
		\EndIf
		\EndFor
	\end{algorithmic}
\end{algorithm}

The \emph{runtime} (also called \emph{optimization time}) of the \oneoneea on a function~$f$ 
is the first point of time~$t$ where a search point of maximal fitness has been created; often 
the expected runtime, \ie, the expected value of this time, is analyzed.  
The \oneoneea with $r=1$ has been extensively studied on simple unimodal problems 
like \[
\onemax(x_1,\dots,x_n)\coloneqq \ones{x},\]  and 
\[
\leadingones(x_1,\dots,x_n)\coloneqq \sum_{i=1}^n \prod_{j=1}^i x_j
\]
but also on the multimodal $\jump_m$ 
function with gap size~$m$ defined as follows:
\[
\jump_m(x_1,\dots,x_n) = 
\begin{cases}
m + \ones{x} & \text{ if $\ones{x}\le n-m$ or $\ones{x}=n$}\\
n-\ones{x} & \text{ otherwise}
\end{cases}
\]
The classical \oneoneea with $r=1$ optimizes these functions in expected time 
$\Theta(n\log n)$, $\Theta(n^2)$ and $\Theta(n^m+n\log n)$, respectively (see, \eg, \cite{DrosteJW02}).

The first two problems are unimodal functions, while \jump for $m\ge 2$ is multimodal and has a local optimum 
at the set of points where $\ones{x}=n-m$. To overcome this optimum, $m$ bits have to flip simultaneously. It 
is well known \citep{DoerrLMNGECCO17} that the time to leave this optimum is minimized at strength~$m$ instead of 
strength~$1$ (see below for a more detailed exposition of this phenomenon). Hence, the \oneoneea would benefit from 
increasing its strength when sitting at the local optimum. The algorithm does not immediately 
know that it sits at a local optimum. However, if there is an improvement at Hamming distance~$1$ then such an improvement 
has probability at least $(1-1/n)^{n-1}/n\ge 1/(en)$ with strength~$1$, 
and the probability of not finding it in $en\ln n$ steps is at most 
\[
(1-1/(en))^{en\ln n}\le 1/n.
\]
Similarly, if there is an improvement that can be reached by flipping $k$ bits simultaneously and the current strength equals~$k$, then 
the probability of not finding it within $((en)^k/k^k) \ln n$ steps is at most 
\[
\left(1- \frac{k^k}{(en)^k}\right)^{((en)^k/k^k)\ln n}\le \frac{1}{n}.
\]
Hence, after $((en)^k/k^k) \ln n$ steps without improvement there is high evidence for that no improvement at Hamming distance~$k$ exists.

We put this ideas into an algorithmic framework by counting the number of so-called unsuccessful steps, \ie, 
steps that do not improve fitness. Starting from strength~$1$, 
the strength is increased from $r$ to~$r+1$ when the counter exceeds the threshold $2((en)^r/r^r) \ln (nR)$ for a parameter $R$ to be discussed shortly. 
Both counter and strength are reset (to $0$ and $1$ respectively) when an improvement is found, \ie, a search point of strictly better fitness. In the 
context of the \oneoneea, the stagnation detection (SD) is incorporated in Algorithm~\ref{alg:sdoneone}. We see that the counter~$u$ is increased 
in every iteration that does not find a strict improvement. However, search points of equal fitness are still accepted as in the classical 
\oneoneea. We note that the strength stays at its initial value~$1$ if finding an improvement does not take longer than 
the corresponding threshold $2en\ln(Rn)$; if the threshold is never exceeded the algorithm behaves identical to the \oneoneea with 
strength~$1$ according to Algorithm~\ref{alg:oneone-classic}.

The parameter $R$ can be used to control the probability of failing to find an improvement at the ``right'' strength. 
More precisely, the probability of not finding an improvement 
at distance~$r$ 
with strength~$r$ is at most 
\[
\left(1- \frac{r^r}{(en)^r}\right)^{(2(en)^r/r^r)\ln (nR)}\le \frac{1}{(nR)^2}.
\]
As shown below in Theorem~\ref{lem:unimodal}, if $R$ is set to the number of fitness values of the underlying function~$f$, \ie, $R=\card{\im(f)}$, 
then the probability 
of ever missing an improvement at the right strength is sufficiently small throughout the run. We recommend at least $R=n$ if nothing is known 
about the range of~$f$, resulting in a threshold of at least $4((en)^r/r^r) \ln(n) $ at strength~$r$.

We also add stagnation detection to the \oplea with self-adjusting mutation rate defined in \cite{DoerrGWYAlgorithmica19} (adapted to maximization 
of the fitness function), where 
half of the offspring are created with strength~$r/2$ and the other half with strength~$2r$; see Algorithm~\ref{alg:sdonelambda}.
Unsuccessful 
mutations are counted in the same way as in Algorithm~\ref{alg:sdoneone}, taking into account that $\lambda$ offspring are used. 
The algorithm can be in two states. 
Unless the counter threshold is reached and a strength increase is triggered, the algorithm behaves the same as the self-adjusting 
\oplea from \cite{DoerrGWYAlgorithmica19} (State~2). If, however, 
the counter threshold $2en\ln(nR)/\lambda$ is reached, then the algorithm changes to the module that keeps increasing the strength until 
a strict improvement is found (State~1). Since it does not make sense to decrease the strength in this situation, all offspring use the same strength
until finally an improvement is found and the algorithm changes back to the original behavior using two strengths for the offspring. 
The boolean variable~$g$ keeps track of the state.
From the discussion of these two algorithms, we see that the stagnation detection consisting of counter for unsuccessful steps, 
threshold, and strength increase also can be added to other algorithms, while keeping their original behavior unless 
the counter threshold it reached.

\begin{algorithm}
	\caption{\ooea with stagnation detection (\sdooea)}
	\label{alg:sdoneone}
	\begin{algorithmic}
		\State Select $x$ uniformly at random from $\{0, 1\}^n$ and set $r_1 \gets 1$.
		\State $u\gets 0$.
		\For{$t \gets 1, 2, \dots$}
		\State Create $y$ by flipping each bit in a copy of $x$ independently with probability~$\frac{r_t}{n}$.
		\State $u\gets u+1$.
		\If{$f(y) > f(x)$}
		\State $x \gets y$.
		\State $r_{t+1}\gets 1$.
		\State $u\gets 0$.
		\ElsIf {$f(y) = f(x)$ \textbf{and} $r_t=1$}
		\State $x \gets y$.

		\EndIf
		\If{$u > 2\left(\frac{en}{r_t} \right)^{r_t}\ln (nR)$} 
		\State $r_{t+1}\gets \min\{r_t+1,n/2\}$.
		\State $u\gets 0$.
		\Else
		\State $r_{t+1}\gets r_t$.
		\EndIf
		\EndFor
	\end{algorithmic}
\end{algorithm}

\begin{algorithm}
\caption{\oplea with two-rate standard bit mutation and stagnation detection (\saonelambdaea)}
\small
\label{alg:sdonelambda}
	\begin{algorithmic}
		\State Select $x$ uniformly at random from $\{0, 1\}^n$ and set $r_1 \gets r^{\text{init}}$.
		  \State $u\gets 0$.
		  \State $g\gets \textit{False}$ (boolean variable indicating stagnation detection)
		\For{$t \gets 1, 2, \dots$}
			\State $u\gets u+1$.

			\If{ $g = \textit{True}$} \label{code:gap}
			    \phase{Stagnation Detection}			
			\For{$i \gets 1, \dots, \lambda$}

\State Create $x_i$ by flipping each bit in a copy of $x$ independently with probability~$\frac{r_t}{n}$.

\EndFor
\State $y \gets \arg\max_{x_i} f(x_i)$ (breaking ties randomly).
			\If{$f(y) > f(x)$}
				\State $x \gets y$.
				\State $r_{t+1}\gets r^{\text{init}}$.
				\State $g \gets \textit{False}$.
				\State $u\gets 0$.
			\Else
			\If{$u > 2\left(\frac{en}{r_t} \right)^{r_t}\ln (nR)/\lambda$} 
      \State $r_{t+1}\gets \min\{r_t+1,n/2\}$.
			\State $u\gets 0$.
			\Else
			\State $r_{t+1}\gets r_t$.
			\EndIf
			\EndIf
			\Else \quad (\ie, $g = \textit{False}$)
						    \phase{Self-Adjusting \oplea}
						\For{$i \gets 1, \dots, \lambda$}

        \State Create $x_i$ by flipping each bit in a copy of $x$ independently with probability~$\frac{r_t}{2n}$ if $i\leq\lambda/2$ and with probability $2r_t/n$ otherwise.

			\EndFor
			\State $y \gets \arg\min_{x_i} f(x_i)$ (breaking ties randomly).
			\If{$f(y) \ge f(x)$}
						\If{$f(y) > f(x)$}
			\State $u\gets 0$.
			\EndIf
			\State $x \gets y$.

			\EndIf
			      \State Perform one of the following two actions with prob.~$1/2$:
			\State \quad -- Replace $r_t$ with the strength that $y$ has been created with.
			\State \quad -- Replace $r_t$ with either $r_t/2$ or $2r_t$, each with probability~$1/2$.
			\State $r_{t+1} \gets \min\{\max\{2, r_t\}, n/4\}$.
			\If{$u > 2\left(\frac{en}{r_t} \right)^{r_t}\ln (nR)/\lambda$} 
			\State $r_{t+1}\gets 2$.
			\State $g\gets \textit{True}$.
			\State $u\gets 0$.
			\EndIf
			\EndIf
		\EndFor
	\end{algorithmic}
\end{algorithm}

\subsection{Mathematical Tools}
We now collect frequently used mathematical tools. The first one is a simple summation formula 
used to analyze the time spent until the strength is increased to a certain value.
\begin{lemma} \label{lem:partial-sum}
For $m<n$, we have $\sum_{ i=1}^{m}\left(\frac{en}{i}\right)^i < \frac{n}{n-m} \left(\frac{en}{m}\right)^m$.
\end{lemma}

\begin{proofwithoutbox} We have 
$\left(\frac{en}{m-i}\right)^{m-i}= \left(\frac{m}{en}\right)^i \left(\frac{m}{m-i}\right)^{m-i} \left(\frac{en}{m}\right)^{m}$ for all~$i~<~m$, so
		\begin{align*}
		\sum_{ i=1}^{m}\left(\frac{en}{i}\right)^i &=\sum_{ i=0}^{m-1}\left(\frac{en}{m-i}\right)^{m-i}
		= \left(\frac{en}{m}\right)^m \sum_{ i=0}^{m-1} \left( \frac{m}{en}\right)^i  \left(1+\frac{i}{m-i}\right)^{m-i}  \\
		& < \left(\frac{en}{m}\right)^m \sum_{ i=0}^{m-1} \left( \frac{m}{n}\right)^i 
		< \frac{n}{n-m} \left(\frac{en}{m}\right)^m.  
		\qquad\Box\end{align*}
\end{proofwithoutbox}

The following result due to Hajek applies to processes with a strong drift towards some target state, resulting 
in decreasing occupation probabilities with respect to the distance from the target. On top of this occupation 
probabilities, the theorem bounds \emph{occupation times}, \ie, the number of steps that the process spends 
in a non-target state over a certain time period.

\begin{theorem}[Theorem~3.1 in~\cite{HajekDrift}]
\label{theo:hajek-occ-times}
Let $X_t$, $t\ge 0$, be a stochastic process adapted to a filtration $\filtt$ on $\R$. Let $a\in \R$. Assume 
for $\Delta_t=X_{t+1}-X_t$ that there are $\eta>0, \delta<1$ and $D>0$ such that 
that 
\begin{enumerate}
\item $\expect{e^{\eta \Delta}\mid \filtt; X_t>a} \le \rho$
\item $\expect{e^{\eta \Delta}\mid \filtt; X_t\le a} \le D$
\end{enumerate}  
If additionally $X_0$ is of exponential type (\ie, $\expect{e^{\lambda X_0}}$ is finite for some~$\lambda>0$) then for any constant $\epsilon>0$
 there exist absolute constants 
$K\ge 0, \delta<1$ such that for all $b\ge a$ and $T\ge 1$ 
\[
\Prob\Bigl(\frac{1}{T}\sum_{t=1}^T \indic{X_t\le b} \le 1-\epsilon-\frac{1-\epsilon}{1-\rho}De^{\eta (a-b)}\Bigr) \;\le\; K\delta^T
\]
\end{theorem} 

\section{Analysis of \sdooea}
\label{sec:upper}
In this section, we study the \sdooea from Algorithm~\ref{alg:sdoneone} in greater detail. We show general upper and lower bounds on multimodal 
functions and then analyze the special case of \jump more precisely. We also show the important result that on unimodal functions, 
the \sdooea with high probability 
behaves in the same way as the classical \ooea with strength~$1$, including the same asymptotic bound on the expected optimization time.

\subsection{Expected Times to Leave Local Optima}

In the following, given a fitness function $f\colon\{0,1\}^n\to\R$, we call the \emph{gap} of the point $x\in \{0,1\}^n$ 
the minimum hamming 
distance to points  with strictly larger fitness function value. Formally,
\begin{align*}
\gap(x)\coloneqq \min\{H(x,y):f(y)>f(x) , y\in \{0,1\}^n\}.
\end{align*}

Obviously, it is not possible to improve fitness 
by changing less than $\gap(x)$ bits 
of the current search point. However, if the algorithm creates a point of $\gap(x)$ 
distance from the current search point $x$, we can make progress with a positive probability. Note that 
$\gap(x)=1$ is allowed, so the definition also covers points that are not local optima.

Hereinafter, $T_x$ denotes the number of steps of \sdooea to find an 
improvement point when the current search point is $x$. Let phase~$r$ consists of all points of time 
where strength~$r$ is used in the algorithm with stagnation counter.
Let \emph{$E_r$} be the event of \textbf{not} finding the optimum by the end of 
phase~$r$, and \emph{$U_r$} be the event of not finding the optimum 
during phases 1 to $r-1$ and finding in phase $r$. In other words, $U_r=E_1\cap \dots \cap E_{r-1} \cap \overline{E_{r}}$.

The following lemma will be used throughout this section. It shows that the probability of 
not finding a search point with larger fitness value in phases of larger strength than the real gap 
size is small; however, by definition phase~$n/2$ is not finished before the algorithm finds an improvement. In the statement of the lemma, recall that the parameter $R$ controls the threshold for the number of unsuccessful steps in stagnation detection.

\begin{lemma} \label{lem:failure-probability}
	Let $x\in\{0,1\}^n$ be the current search point of the \sdooea on a pseudo-boolean fitness function $f\colon\{0,1\}^n\to \R$ 
	and let $m=\gap(x)$. 
	Then 
	\begin{align*}
		\prob{E_r} \le     
		\begin{cases}
		\frac{1}{(n R)^{2}} & \text{ if } m \le r < n/2\\
		0 & \text{if } r=n/2.
		\end{cases}
	\end{align*}
\end{lemma}

\begin{proof}
The algorithm spends $2e^rn^r/r^{r} \ln (nR)$ steps at strength~$r$ until it increases the counter. Then, the 
		probability of not improving at strength $r\ge m$ is at most
		\begin{align*}
		\prob{E_r} & = \left(1-\left(1-\frac rn\right)^{n-m}\left(\frac rn\right)^m\right)^{2e^rn^r/r^{r}  \ln (nR)} \le \frac{1}{(n R)^{2}}.
		\end{align*}
		
	During phase $n/2$, the algorithm does not increase the strength, and it continues to mutate each bit 
	with probability of $1/2$. As each point on domain is accessible in this phase, the probability of eventually failing to find the improvement 
	is~$0$.
\end{proof}

We turn the previous observation into a general lemma on improvement times.

\begin{theorem} \label{theo:gapx}
Let $x\in\{0,1\}^n$ be the current search point of the \sdooea 
 on a  pseudo-boolean function $f\colon\{0,1\}^n\to \R$.
	Define $T_x$ as the time to create a strict improvement 
	 and $L_{x,k} \coloneqq \expect{T_x}$ if $\gap(x)=k$. Then,  using 
	$m=\min\{k,n/2\}$, we have for all $x$ with $\gap(x)=k$ that
	\[\left(\frac{en}{m}\right)^{m}\left(1-\frac {m^2}{n-m}\right) < L_{x,k}\le 2\left(\frac{en}{m}\right)^{m}\left(1+\frac {5m}{n}\ln(nR)\right).\]
\end{theorem}

\begin{proofwithoutbox}
	Using the law of total probability with respect to the events $U_i$ defined above, we have
	\begin{align} \label{eq:lte}
	\expect{T_x} = \sum_{i=1}^{n/2} \expect{T_x\mid U_i} \prob{U_i}.
	\end{align}
	
		Note that the algorithm does not increase the strength to more than $n/2$.
	By assuming that the algorithm pessimistically does not find a better point 
	for $r<m$, we can bound the formula \eqref{eq:lte} as follows:
	\begin{align*}
	\expect{T_x} 
	&< \underbrace{\expect{T_x \mid U_m}}_{=:S_1} +
	\underbrace{ \sum_{i=m+1}^{n/2} \expect{T_x \mid U_i} \prob{U_i}}_{=:S_2}.
	\end{align*}

	Regarding $S_1$, it takes $\sum_{ i=1}^{m-1} 2(en/i)^i \ln (n R)$ steps until the \sdooea 
	increases the strength to~$m$. When the mutation probability is $m/n$, within an expected number 
	of  $\left((m/n)^m(1-m/n)^{n-m}\right)^{-1}$ steps, a better point will be found. Thus, by using Lemma~\ref{lem:partial-sum}, we have
	\begin{align*}
	\expect{T_x \mid U_m} & \le \sum_{i=1}^{m-1} 2 \left(\frac{en}{i}\right)^i \ln (n R) + \frac{1}{(m/n)^m(1-m/n)^{n-m}} \\
	& < 2\frac{n}{n-m+1} \left(\frac{en}{m-1}\right)^{m-1}  \ln (n R) + \left(\frac{en}{m}\right)^{m} \\
	& < \left(\frac{en}{m}\right)^{m}\left(1+\frac {5m}{en}\left(1+\frac 1{m-1}\right)^{m-1}\ln(nR)\right) \\
	 & \le \left(\frac{en}{m}\right)^{m}\left(1+\frac {5m}{n}\ln(nR)\right).
	\end{align*}
	
	In order to estimate $S_2$, if $m=n/2$, the value of $S_2$ equals zero. Otherwise, by using Lemma~\ref{lem:failure-probability}, $\prob{U_i}< \prod_{j=m}^{i-1}\prob{E_j}<n^{-2(i-m)}$  for $i\ge m+1$ 
	since $R \ge 1$.  We compute
	\begin{align*}
	\sum_{i=m+1}^{n/2} \expect{T_x \mid U_i} \prob{U_i}  & \le \sum_{i=m+1}^{n/2} O\left(\left(\frac{en}{i}\right)^i\ln (nR)\right) n^{-2(i-m)} \\
	&= \ln (nR)\sum_{i=m+1}^{n/2} O\left(\left(\frac ei\right)^i n^{2m-i}\right) 
	 = o( (en/m)^m ).
	\end{align*}
	
Altogether, we have $ \expect{T_x} \le \left(\frac{en}{m}\right)^{m}\left(1+\frac {5m}{n}\ln(nR)\right)+o( (en/m)^m ).$
	
	Moreover,  
	the expected number of iterations for finding an improvement is at least
	$ p^{-m} \left(1-p\right)^{-(n-m)} $ for any mutation rate $p$. Using 
	the same arguments as in the analysis of the \ooea on \jump in~\cite{DoerrLMNGECCO17}, since $\frac{m}{n}$ is the unique minimum point in the interval $[0,1]$,
	\begin{align*}
	\expect{T_x}\geq\left(m/n\right)^{-m} \left(1-m/n\right)^{-(n-m)} \geq (en/m)^m \left(1-\frac{m^2}{n-m}\right).
	\qquad\Box\end{align*}
\end{proofwithoutbox}

We now present the above-mentioned 
important ``simulation result'' implying that on unimodal functions, 
the stagnation detection of \sdooea is unlikely ever to trigger a  
strength increase during its run. Moreover, for a wide range of runtime bounds 
obtained via the fitness level method \citep{WegenerMethods}, 
we show that these bounds transfer to 
the \sdooea up to vanishingly small error terms. The proof carefully estimates 
the probability of the strength ever exceeding~$1$.

\newcommand{\tsd}{T_{\text{sd}}}
\newcommand{\tclassic}{T_{\text{classic}}}

\begin{lemma} \label{lem:unimodal} 
	Let $f\colon\{0,1\}^n\to \R$ be a unimodal function and consider the \sdooea with $R\ge\card{\im(f)}$. Then, with probability $1-o(1)$, the 
	\sdooea never increases the strength  and behaves stochastically like the \oneoneea before finding an optimum of~$f$.

	Denote by $\tsd$ and $\tclassic$ the runtime of the \sdooea and 
	the classical \oneoneea with strength~$1$ on~$f$, respectively. If $U$ is an upper bound 
	on $\expect{T_{\tclassic}}$ obtained by summing up worst-case expected waiting times for improving 
	over all fitness values in $\im(f)$, then 
	\[
	\expect{\tsd} \le  U + o(1).
	\]
	
	The same statements hold with \sdooea replaced with \saonelambdaea, and \oneoneea replaced with 
	the self-adjusting \onelambdaea without stagnation detection.
\end{lemma}

\begin{proof}
  We let the random set~$W$ contain the search points from which the \sdooea does not find an improvement within phase~$1$ (\ie, while $r_t=1$).
	As above, $E_1$ denotes the probability of not finding an improvement within phase~$1$. 
	As on unimodal functions, the gap of all points is~$1$, we have by Lemma~\ref{lem:failure-probability} that 
	$\prob{E_1}\le \frac{1}{(Rn)^2}$. This argumentation holds for each improvement that has to be found. 
	Since at most $\card{\im(f)}\le R$ improving steps happen before 
	finding the optimum, by a union bound the probability of the \sdooea ever increasing the strength beyond~$1$ is at most 
	$R\frac{1}{(Rn)^2} = o(1)$, which proves the first claim of the lemma.
	
	To prove the second claim, we consider all fitness values $f_1<\dots<f_{\card{\im(f)}}$ in increasing 
	order and sum up upper bounds on the expected times to improve from each of these fitness values. Under the 
	condition that the strength is not increased before leaving a fitness level, the worst-case time to leave a level (over all 
	search points with the same fitness value) is clearly 
	not increased. Hence, 
	we 
	bound the expected optimization time of the \sdooea from above 
	by adding the waiting times on all fitness levels for the \oneoneea, which is given by $U$, and the expected times 
	spent to leave the points in~$W$; formally,  
	\[
	\expect{\tsd} \le U + \sum_{x\in W}\expect{T_x}. 
	\]
	
	%

	Each point in~$\im(f)$ contributes with probability $\prob{E_1}$ to~$W$. Hence , 
	$\expect{\card{W}}\le \im(f)\prob{E_1} \le R\prob{E_1}$. 
	As on unimodal functions, the gap of all points is~1, by 
	Lemma \ref{lem:failure-probability}, we have $\prob{U_i}< \prod_{j=1}^{i-1}\prob{E_j}<n^{2-2i}$. Hence,
	\begin{align*}
	\expect{\tsd} 
	& < U +\sum_{x\in W}\expect{T_x} \\
	& < U +R\cdot \prob{E_1}\sum_{i=1}^{n/2} \expect{T_x \mid U_i} \prob{U_i} \\
	& < U +R\cdot (nR)^{-2}\sum_{i=1}^{n/2} 
	O\left(\left(\frac{e}{i}\right)^i n^{2-i} \ln (nR)\right).
	\end{align*}
	
	The second term is $o(1)$, hence
	\begin{align*}
	\expect{T} \le U+o(1).
	\end{align*}
	as suggested.
	
	All the arguments are used in the same way 
	with respect to the  \saonelambdaea and its original formulation 
	without stagnation detection.
\end{proof}

\subsection{Analysis on \jump}

It is well known  that strength~$1$ for the \oneoneea leads to an expected runtime of $\Theta(n^m)$ on $\jump_m$ if $m\ge 2$ \citep{DrosteJW02}. 
The asymptotically 
dominating term comes from the fact that  $m$ bits must flip simultaneously to leave the local optimum at $n-m$ one-bits. To 
minimize the time for such an escaping mutation, mutation rate $m/n$ is optimal \citep{DoerrLMNGECCO17}, leading to an expected time of 
$(1+o(1))(n/m)^m (1-m/n)^{m-n}$ 
 to  optimize $\jump$, which is $\Theta((en/m)^m)$ for $m=o(\sqrt{n})$. However, a static rate of $m/n$ cannot be chosen without knowing the gap size~$m$. 
Therefore, different heavy-tailed mutation operators have been proposed for the \oneoneea \citep{DoerrLMNGECCO17, FriedrichQWGECCO18},  which
 most of the time choose strength~$1$ 
but also use strength~$r$, for arbitrary $r\in\{1,\dots,n/2\}$ 
with at least polynomial probability. This results 
in optimization times on \jump of $\Theta( (en/m)^m \cdot p(n))$ for some small polynomial $p(n)$ (roughly, $p(n)=\omega(\sqrt{m})$ 
in  \cite{DoerrLMNGECCO17} and $p(n)=\Theta(n)$ in \cite{FriedrichQWGECCO18}). Similar polynomial overheads occur with hypermutations 
as used in artificial immune systems \citep{CorusEtAlFastAISPPSN18}; in fact such overheads cannot be completely avoided 
with heavy-tailed mutation operators, as proved in \cite{DoerrLMNGECCO17}. We also remark
 that \jump can be optimized faster than $O((en/m)^m)$ if crossover is used \citep{WhitleyVHM18,RoweAFOGA19},  
by simple estimation-of-distribution algorithms \citep{DoerrGECCO19} or specific black-box algorithms \citep{BuzdalovDoerrKeverECJ16}. 
In addition, the optimization time of 
$n^{(m+1)/2}e^{O(m)}m^{-m/2}$ is shown for the  (1+$(\lambda,\lambda)$)~GA to optimize \jump with $2<m<n/16$ in \cite{AntipovDoerrKaravaev2020}.
All of this 
 is outside the scope of this study 
that concentrates on mutation-only algorithms.

We now state our main result, implying that the \sdooea achieves an asymptotically optimal 
runtime on $\jump_m$ for $m=o(\sqrt{n})$, hence being faster than the heavy-tailed mutations mentioned above. Recall that this does not 
come at a significant extra cost for simple unimodal functions like \om according to Lemma~\ref{lem:unimodal}.

\begin{theorem}
\label{theo:jump}
	Let $n\in \N$. For all $2\le m=O(n/\ln n)$, the expected runtime $\expect{T}$ of the \sdooea on $\jump_m$ satisfies 
	\[
	\Omega\left( \left(\frac{en}{m}\right)^m \left(1-\frac{m^2}{n-m}\right) \right) \leq \expect{T} \le O\left(\left(\frac{en}{m}\right)^m\right). 
	\]
\end{theorem}
                                                                                    
\begin{proofwithoutbox}
	It is well known that the \oneoneea with mutation rate $1/n$ finds the 
	optimum of the $n$-dimensional OneMax function in an expected number of at most $en\ln n-O(n)$ iterations. 
	
	Until reaching the plateau consisting of all points of $n-m$ one-bits, \jump is equivalent to \onemax; hence, 
	according to Lemma~\ref{lem:unimodal}, the expected time until \sdooea reaches the plateau is at most $O(n\ln n)$ (noting that this bound was obtained 
	via the fitness level method).
	
	Every plateau point $x$ with $n-m$ one-bits satisfies $\gap(x)=m$ according to the definition of \jump. 
	Thus, 	using Theorem~\ref{theo:gapx}, the algorithm finds the optimum within expected time
	\[
	\Omega\left( \left(\frac{en}{m}\right)^m \left(1-\frac{m^2}{n-m}\right) \right) \leq \expect{T_x} \le O\left(\left(\frac{en}{m}\right)^m\right).\]
	This 
	 dominates the expected time of the algorithm before the plateau point.
	
	Finally, 
	\begin{align*}
	\Omega\left( \left(\frac{en}{m}\right)^m \left(1-\frac{m^2}{n-m}\right) \right) \leq \expect{T} \le O\left(\left(\frac{en}{m}\right)^m\right).
	\qquad\Box\end{align*}
\end{proofwithoutbox}

It is easy to see (similarly to the analysis of Theorem \ref{theo:jump}) that for all $m=\Theta(n)$, the expected runtime $\expect{T}$ of the \sdooea on $\jump_m$ satisfies $\expect{T} =  O\,\left(\left(\frac{en}{m}\right)^m\ln n\right)$.

\subsection{General Bounds}

The \jump function only has one local optimum that usually has to be overcome on the way to the global optimum. We generalize 
the previous analysis to functions that have multiple local optima of possibly different gap sizes. As a special case, we can 
asymptotically recover the expected runtime on the \leadingones function in Corollary~\ref{cor:leadingones}.

\begin{theorem}
 \label{theo:general}
	The expected runtime of the  \sdooea on a pseudo-Boolean fitness function $f$ is at most 
	\[ 
	\expect{T \mid V_1,\dots,V_{n}} = O \left( \sum_{k=1}^{n} V_k L_k \right),
	\]
	where $V_k$ is the number of points~$x$ of $\gap(x)=k$ visited by the algorithm and $L_k:=\max\{L_{x,k}\mid x\in\{0,1\}^n\,\wedge\,\gap(x)=k\}$ with $L_{x,k}$ as defined in Theorem \ref{theo:gapx}.
   Moreover,
	\[
	\expect{T} = O\left( \sum_{k=1}^{n} \expect{V_k} L_k \right),
	\]
\end{theorem}

\begin{proof}
	The \sdooea visits a random trajectory of search points $\{x_0,x_1,x_2,\dots,x_m=x^*\}$ in order to find an optimum point~$x^*$.

	For any search point $x$ with $\gap(x)=k$, the expected time to find a better search point when $r\le m$ is $\expect{T_x} = L_k $
	according to Theorem~\ref{theo:gapx}.
	
	Also, we have
	$T=T_{x_1}+T_{x_2}+\dots+T_{x_m}=\sum_{k=1}^{n} V_k \cdot (T_x \mid \gap(x)=k)$. 
	 Therefore, as the strength $r$ is reset to~$1$ after each improvement, we have
	\[
	\expect{T\mid V_1,\dots,V_n} = O\left( \sum_{k=1}^{n} V_k L_k \right),
	\]
	which proves the first statement of this theorem. The second follows by the law of total expectation.
\end{proof}

\begin{corollary} \label{cor:leadingones}
	The expected runtime of the \sdooea on \leadingones is at most $O(n^2)$.
\end{corollary}

\begin{proof}
	On \leadingones, there are at most $n$ points of gap size~$1$, so according to 
	Theorem~\ref{theo:general}, the expected runtime is $O(n^2)$.
\end{proof}

Corollary~\ref{cor:leadingones} can be also inferred from Lemma~\ref{lem:unimodal} since \leadingones is unimodal and the $O(n^2)$ bound 
was inferred via the fitness level method.

We finally specialize Theorem~\ref{theo:general} into a result for the well-known \trap function \cite{DrosteJW02} that 
is identical for $\onemax$ except for the all-zeros string that has optimal fitness $n+1$. 
We obtain a bound of $2^{\Theta(n)}$ instead of the $\Theta(n^n)$ bound for the classical \oneoneea. 
The base of our result is somewhat larger than for the fast~GA from 
\cite{DoerrLMNGECCO17}; however, it is still close to the $2^n$ bound  that would be obtained by uniform search.

\begin{corollary}
	The expected runtime of \sdooea on \trap is at most $O(2.34^n\ln n)$.
\end{corollary}

\begin{proof}
	On \trap, 
	there are one point of gap size~$n$ and $O(n)$ points with gap size of $1$. So 
	according to Theorem \ref{theo:general}, the expected runtime is $O\left((2.34)^{n}\ln n\right)$.
\end{proof}

\section{An Example Where Self-Adaptation  Fails}
\label{sec:needhighmut}

While our previous analyses have shown the benefits of the self-adjusting scheme, in particular 
highlighting stagnation detection on multimodal functions, it is clear that our scheme also has 
limitations. In this section, we present an example of a pseudo-Boolean function where stagnation 
detection does not help to find its global optimum in polynomial time; moreover, the function 
is hard for other self-adjusting schemes since measuring the number of successes does not hint 
on the location of the global optimum. In fact, the function demonstrates a more general effect 
where the behavior is very sensitive with respect to choice of the the mutation probability. More precisely, 
a plain \oneoneea with mutation probability~$1/n$ with overwhelming probability 
gets stuck in a local optimum from which 
it needs exponential time to escape while the \oneoneea with mutation probability $2/n$ and also 
above finds the global 
optimum in polynomial time with overwhelming probability. Since the function is unimodal except at the local 
optimum, our self-adjusting \oneoneea with stagnation detection fails as well. 

To the best of our knowledge, a phase 
transition with respect to the mutation probability  
where an increase by a small constant factor leads from exponential to polynomial 
optimization time 
has been unknown in the literature of runtime analysis so far and may be of independent interest. We are 
aware of opposite phase transitions on monotone functions \citep{LenglerPPSN18} where increasing the mutation rate
is detrimental; however, we feel that our function and the general underlying construction principle 
are easier to understand 
than these specific monotone functions.

The construction of our function, called \needhighmut, is based on a general principle that was introduced
 in \cite{WittCEC03} to show the benefits of populations and was subsequently applied in \cite{JansenWiegandECJ04} 
to separate a coevolutionary variant of the \oneoneea from the standard \oneoneea. Section~5 of the latter paper 
also beautifully describes the general construction technique that involves creating two differently pronounced 
gradients for the algorithms to follow. 
Further applications are given in \cite{WittECJ06}  and \cite{WittTCS08}   
to show the benefit of populations in elitist and non-elitist EAs.  Also \cite{RohlfshagenLehreYaoGECCO09} use very similar 
construction technique 
for their \textsc{Balance} function that is easier 
to optimize in frequently changing than slowly changing environments; however, they did not seem to be aware 
that  their approach resembles earlier work from the papers above.

We now describe the construction of our function \needhighmut. The crucial observation is that strength $1$ (\ie, 
probability~$p=1/n$)  
makes it more likely to flip exactly one specific bit than strength~$2$ -- in fact strength~$1$ 
is asymptotically 
optimal since the probability of flipping one specific bit is $p(1-p)^{n-1}\approx p e^{-pn}$, which is maximized 
for $p=1/n$. However, to flip specific two bits,  
which has probability $p^2(1-p)^{n-2} \approx p^2 e^{-pn}$, the choice $p=2/n$ 
is asymptotically optimal and clearly better than $1/n$. Now, given a hypothetical time span of $T$, 
we expect approximately $T_1(p)\coloneqq T pe^{-p/n}$ specific one-bit and $T_2(p)\coloneqq T p^2e^{-p/n}$ specific two-bit 
flips. Assuming the actual numbers to be concentrated and just arguing with expected values, 
we have $T_1(1/n) \gg T_2(1/n)$ but $T_2(2/n) \gg  T_1(2/n)$, \ie, there will be 
considerably more two-bit flips 
at strength~$2$ than at strength~$1$ and considerably less $1$-bit flips. The fitness function will account for this. It  
leads to a trap at a local optimum if a certain number of one-bit flips is exceeded before a certain minimum number of
 two-bit flips has happened; however, if the number 
of one-bit flips is low enough before the minimum number of two-bit flips has been reached, the process is 
on track to the global optimum.

We proceed with the formal definition of \needhighmut, making these ideas precise and 
overcoming technical hurdles. Since 
we have at most $n$ specific one-bit flips but a specific two-bit flip 
is already by a factor of $O(1/n)$ less likely than a one-bit flip, we will work with two-bit flips 
happening in small blocks of size~$\sqrt[4]{n}$, leading to a probability of roughly $n^{-3/2}$ 
for a two-bit flip in a block. 
In the following, we will imagine a bit string $x$ of length~$n$ 
as being split into a prefix $a\coloneqq a(x)$ of length~$n-m$ and a suffix $b\coloneqq b(x)$ of 
length $m$, where $m$ still has to be defined.
 Hence, $x=a(x)\circ b(x)$, where $\circ$ denotes the concatenation.

The prefix~$a(x)$ is called \emph{valid} if it is of the form 
$1^i 0^{n-m-i}$, \ie, $i$ leading ones and $n-m-i$ trailing zeros. The prefix fitness 
$\pre(x)$ of a string~$x\in\{0,1\}^n$ with valid prefix 
$a(x)=1^i0^{n-m-i}$ equals just~$i$, the 
number of leading ones. The suffix consists of $\lceil \frac{2}{3}\xi\sqrt{n}\rceil$, where $\xi\ge 1$ is a parameter of the 
function, consecutive blocks of 
$\lceil n^{1/4}\rceil$ bits each, altogether $m\le \xi\frac{2}{3}n^{3/4}=o(n)$ bits. 
Such a block is called \emph{valid} if it contains either~$0$ or $2$ one-bits; moreover, it is called \emph{active} if it contains~$2$ 
and \emph{inactive} if it contains~$0$ one-bits. A suffix where all blocks are valid and where all blocks following 
first inactive block are also inactive is called valid itself, and the suffix fitness $\suff(x)$ of a 
string~$x$ with valid suffix~$b(x)$  
is the number of leading active blocks before the first inactive block. Finally, we call a string $x\in\{0,1\}^n$ valid 
if both its prefix and suffix are valid.

Our final fitness function is a weighted combination of $\pre(x)$ and $\suff(x)$. We define 
for $x\in\{0,1\}^n$, where $x=a\circ b$ with the above-introduced $a$ and~$b$, 
\begin{align*}
& \needhighmut_\xi(x)\coloneqq \\
& \begin{cases}
 n^2 \suff(x) + \pre(x) & \text{\quad if $\pre(x)\le \frac{9(n-m)}{10}$ $\wedge$ $x$ valid}\\
 n^2 m + \pre(x) + \suff(x) - n - 1 \hspace{-2ex}& \text{\quad if $\pre(x)>\frac{9(n-m)}{10}$ $\wedge$ $x$ valid}\\
- \onemax(x) & 
  \text{\quad otherwise.} 
\end{cases}
\end{align*}

We note that all search points in the second case have a fitness of at least~$n^2 m - n - 1$, which 
is bigger than $n^2(m-1) + n$, an upper bound on the fitness of search points that fall into the first case 
without having $m$ leading active blocks in the suffix. Hence, search points~$x$ where $\pre(x)=n-m$ and 
$\suff(x)=\lceil \frac{2}{3}\xi\sqrt{n}\rceil$ represent local optima of second-best overall fitness. The set of global optima 
equals the points where $\pre(x)=9(n-m)/10$ and $\suff(x)=m$, which implies that $(n-m)/10=\Omega(n)$ bits have to be flipped 
simultaneously to escape from the local toward the global optimum.

The parameter $\xi\ge 1$ controls the target strength that allows the algorithm to find the global optimum 
with high probability. In the simple setting  $\xi=1$, strength~$1$ 
usually leads to the local optimum first while strengths above $2$ usually lead directly to the global optimum. Using larger 
$\xi$ increases the threshold for the strength necessary to find the global optimum instead of being trapped in the local
one.

We now formally show with respect to different algorithms that \needhighmut is challenging to optimize 
without setting the right mutation probability in advance. We start with an analysis of the classical \ooea, 
where we for simplicity only show the negative result for $p=1/n$ even though it would even hold 
for $\xi/n$.

\begin{theorem}
\label{theo:plainooea-needhighmut}
Consider the plain \ooea with mutation probability~$p$ on $\needhighmut_\xi$ for a constant $\xi\ge 1$. 
If $p=1/n$ then with probability $1-2^{-\Omega(n)}$, its optimization time is $n^{\Omega(n)}$. If $p = (c\xi)/n$ for any constant 
$c\ge 2$ then the optimization time is $O(n^2)$ with probability $1-2^{-\Omega(\sqrt{n})}$.
\end{theorem}
\begin{proof}
It is easy to see (similarly to the analysis of the \textsc{SufSamp} function from \cite{JansenDW05}) 
that the first valid search point (\ie, search point of non-negative fitness) has both  $\pre$- and $\suff$-value 
value of at most~$n^{1/3}$ with probability $2^{-\Omega(n^{1/3})}$. This follows from the fact that the function is symmetric on 
invalid search points and that from each level set of $i$ one-bits, only $O(1)$ search points are valid. In the following, 
we tacitly assume that we have reached a valid search point of the described maximum $\pre$- and $\suff$-value and note that 
this changes the required number of improvements to reach local or global maximum only by a $1-o(1)$ factor. For readability 
this factor will not be spelt out any more.

We prepare the main analysis by bounding 
the probability of a mutation being accepted after a valid search point has been reached. Even if a mutation changes up to $o(n)$ consecutive bits of the prefix 
or suffix, it must maintain $n-o(n)$ prefix bits in order to result in a valid search points. Hence, the probability of an accepted step 
at mutation probability $c/n$ (valid for any constant~$c$) is at most $(1-c/n)^{n-m-o(n)} = (1+o(1)) e^{-c}$. Steps flipping $\Omega(n)$ consecutive bits 
have probability $n^{-\Omega(n)}$ and are subsumed by the failure probabilities stated in this theorem. Clearly, 
the probability of a accepted step is at least $(1-1/n)^{n} = (1-o(1)) e^{-c}$. 

Using this knowledge of accepted steps, we shall now prove the statement for $p=1/n$. The probability of improving the $\pre$-value 
is at least $e^{-1}/n$ since 
it is sufficient to flip the leftmost zero of the prefix to~$1$. 
In a phase of length $\frac{11}{10}e mn$ steps, there are at least $m$ prefix-improving mutations
with probability $1-2^{-\Omega(n)}$ by Chernoff bounds. All these improve the function value and are accepted unless the $\suff$-value 
increases to~$m$ before the $\pre$-value exceeds $9n/10$. 

The probability of improving the leftmost inactive block of the suffix by~$1$ is at most 
$\binom{n^{1/4}}{2} \frac{1}{n^2} e^{-1}(1+o(1)) \le (1+o(1)) (e^{-1}/2) n^{-3/2}$ 
since it is necessary to flip two zeros into ones and to have an accepted mutation. 
By the same reasoning, steps that activate $k=o(n)$ blocks simultaneously  
have a probability of at most $(1+o(1)) (e^{-1}/2 n^{-3/2})^k$. We consider a phase of~$s\coloneqq \frac{11}{10}emn$ 
steps and 
bound the number of number of accepted steps increasing the $\suff$-value by~$k$ 
by applying Chernoff bounds since this number if bounded by a binomial distribution with 
parameter $s$ and $p_k\coloneqq (1+o(1)) (e^{-1}/2 n^{-3/2})^k$. Hence, the number 
of accepted steps activating one suffix block in  
in $\frac{11}{10}emn \le \frac{11}{10}en^2$ steps is less than $\frac{3}{5}\sqrt{n}$ 
with probability $1-2^{-\Omega(\sqrt{n})}$. The expected number of accepted steps activating 
$k\ge 2$ suffix blocks is already $O(n^{-1/2})$, and by Chernoff bounds the actual number is 
at most $n^{1/3}$ with probability $1-2^{-\Omega(n^{1/3})}$. Hence, by a union bound 
over~$k\in\{2,\dots,n^{1/9}\}$, the steps adding more than one valid suffix block 
increase the $\suff$-value by at most $n^{1/3+1/9}=n^{4/9}$ with probability~$1-2^{-\Omega(n^{1/3})}$. 
Steps adding $k>n^{1/9}$ valid blocks have probability $O(2^{-\Omega(n^{1/9})})$ and are subsumed 
by the failure probability. 
If none of the failure events occurs, the total increase of the $\suff$-value is at most $\frac{3}{5}\sqrt{n}+n^{4/3} < \frac{2}{3}\sqrt{n}$. 
Also, with probability $1-2^{-\Omega(\sqrt{n})}$, the $\pre$-value decreases by altogether at most $O(\sqrt{n})$ in the $O(\sqrt{n})$ 
mutations that improve the suffix, which can be subsumed in a lower-order term in the above analysis of 
$\pre$-improving steps.

Altogether, with overwhelming probability $1-2^{-\Omega(n^{1/9})}$ the prefix is optimized before the suffix. The probability 
of reaching the global optimum from the local one is $n^{-\Omega(n)}$ since it is necessary to 
flip $m/10$ bit simultaneously to leave the local optimum. In a phase of $n^{c'n}$ steps for a
sufficiently small constant $c'$ this does not happen with probability $1-2^{-\Omega(n)}$. This completes the proof 
of the statement for the case~$p=1/n$.

For $p=c/n$, where $c\ge 2\xi$, we argue similarly with inverted roles of prefix and suffix.
 The probability 
of activating a block in the suffix is at least 
$(1-o(1)) ((c^2/2) e^{-c} n^{-3/2})$ now. 
 In a phase 
of $(7/4) \xi (e^2/c^2)  mn$ steps, we expect $(7/8) \xi \sqrt{n}$ activated blocks and with overwhelming probability 
we have at least $\frac{2}{3}\xi\sqrt{n}$ such blocks. The probability of improving the $\pre$-value by~$k$  
is only $(1+o(1))c e^{-c}/n^k$, amounting to an expected number of 
  improvements by~$1$ of $(1+o(1))(7/4) (\xi/c) m n^{1-k} = 
	(1+o(1))(7/4) (\xi/c)  n^{2-k} \le (1+o(1)) (7/8)n^{2-k} $ since $c\ge 2\xi$, 
and, using similar Chernoff and union bounds as above, the probability of 
at least $(9/10)m$ $\pre$-improving steps in the phase is $2^{-\Omega(n^{1/3})}$. 
\end{proof}

The previous analysis can be transferred to the \sdooea with stagnation detection, showing that 
this mechanism does not help to increase the success probability significantly compared to the plain \ooea with $p=1/n$. The proof shows that the \sdooea with high probability does not behave differently from the \ooea. The only
major difference is visible after reaching the local optimum 
of \needhighmut, 
where stagnation detection kicks in. This results in the 
bound $2^{\Omega(n)}$ in the following theorem, compared to 
$n^{\Omega(n)}$ in the previous one.

\begin{theorem}
\label{theo:sdooea-fail}
With probability at least $1-O(1/n)$, the \sdooea needs at least $2^{\Omega(n)}$ steps to optimize 
$\needhighmut_\xi$ for $\xi\ge 1$.
\end{theorem}

\begin{proof}
We assume that the parameter $\card{R}$ of the algorithm is set to at least~$n$ and 
follow the analysis of the case $p=1/n$ from the proof of Theorem~\ref{theo:plainooea-needhighmut}. In a phase of 
$\frac{11}{10}e mn$ steps, there are at least $m$ $\pre$-improving mutations (having probability 
at least $1/(en)$ each) 
with probability $1-2^{-\Omega(n)}$ by Chernoff bounds. For each of these improving mutations, the probability 
that it does not happen within the threshold of $en\ln(n\card{R}) \ge en\ln(n^2)$ iterations is at most 
$(1-1/(en))^{en\ln(n^2)} \le 1/n^2$. By a union bound, the probability that at least one of the mutations does not happen 
within this number of iterations is at most~$1/n$. Together with the analysis of the number of $\suff$-increasing mutations, 
this means that the strength stays at~$1$ until the local optimum is reached, and that the local optimum is reached first, 
with probability at least 
$1-O(1/n)$.

Leaving the local optimum requires a mutation flipping at least $m/10=\Omega(n)$ bits simultaneously. As already analyzed in Theorem~\ref{theo:gapx}, 
even at optimal strength this requires $2^{\Omega(n)}$ steps with probability 
$1-2^{-\Omega(n)}$. Taking a union bound over all failure probabilities completes the proof.
\end{proof}
 
Finally, we also show that the self-adaptation scheme of the \saonelambdaea does not help to concentrate the  
mutation rate on the right regime for $\needhighmut_\xi$ if $\xi$ is a sufficiently large constant and $\lambda$ 
is not too large. This still applies 
in connection with stagnation detection.

\begin{theorem}
\label{theo:onelambda-fail-needhighmut}
Let $\xi$ be a sufficiently large constant and assume $\lambda=o(n)$ and $\lambda=\omega(1)$. Then 
with probability at least $1-O(1/n)$, the \saonelambdaea with stagnation detection (Algorithm~\ref{alg:sdonelambda}) 
needs at least $2^{\Omega(n)}/\lambda$ generations to optimize 
$\needhighmut_1$.
\end{theorem}

The proof of this theorem uses more advanced techniques, more 
precisely Theorem~\ref{theo:hajek-occ-times} to analyze the 
distribution of mutation strength in the offspring over time. This 
technique allows us that only a small constant fraction of steps 
uses strength that are more beneficial for the suffix than the prefix.

\begin{proof}
The idea is to show that the strength has a drift towards its minimum and then apply Theorem~\ref{theo:hajek-occ-times} 
to bound the number of steps at which a mutation rate is taken that could be beneficial. Then, since most of the steps 
use small mutation rates, the prefix is optimized before the suffix with high probability and a local optimum reached.

To make these ideas precise, we pick up and extend the analysis of the acceptance and improvement probabilities from 
Theorem~\ref{theo:plainooea-needhighmut}. Hence (with respect to the creation of a single offspring):\begin{itemize}
\item The probability of accepting a mutation  at strength $r=o(n)$ 
is $(1\pm o(1)) e^{-r}$ since only $o(n)$ bits flip with probability $1-e^{-\omega(r)}$ and 
$(1-o(1))n$ bits have to be preserved (not flipped) with 
probability $1-2^{-\Omega(n)}$. At strengths $r=\Omega(n)$ the probability 
of improving the \pre-value by $m/2$ is $2^{-\Omega(n)}$ since $m/2$ consecutive 
bits have to be set to~$1$; otherwise, at least $m/2$ bits must preserved, 
which has probability at most $e^{-\Omega(r)}$. 
\item   the probability of improving the $\pre$-value by $k=o(n)$ is 
$(1\pm o(1)) (r/n)^k e^{-r}$.
\item  the probability of improving the $\suff$-value by $k=o(n)$ is 
$(1\pm o(1)) (r/n)^k e^{-r}$.
\end{itemize}
Clearly, the probability that at least one out of $\lambda$ offspring is improving the function value is at most $\lambda$ times 
as large. 
Since we have $\lambda=o(n)$ offspring and each improvement has probability $p_i=O(1/n)$, the probability at having at least one 
improving offspring is at least $1-(1-p_i)^{\lambda} = 1-(1-(1-o(1))\lambda p_i$, hence also by a factor at least $(1-o(1))\lambda$ larger.

Using these bounds on the acceptance and improvement probabilities, we now use 
ideas similar to the analysis of the near region in \cite{DoerrGWYAlgorithmica19} to show 
a drift of the strength towards small values. We discuss several cases:

{$r_t\le (\ln\lambda)/4$}: then the probability of creating a copy of the parent at strength~$r_t/2$ is 
at least $(1-o(1)) e^{-(\ln \lambda)/8} = (1-o(1))\lambda^{-1/8}$. This probability is 
by a factor $(1-o(1)) e^{4}$ smaller at strength $2r_t$. Using Chernoff bounds and exploiting $\lambda=\omega(1)$ we have that 
with probability~$1-o(1)$,  
the number of copies produced at strength $r_t/2$ is by a constant factor larger than the one produced 
at strength~$2r_t$, and there is at least one copy produced from strength~$r_t/2$. Due to the uniform choice of the 
individual adjusting the strength in case of ties, the probability of increasing the strength is at most $1/2-\epsilon$ for some constant~$\epsilon>0$. 

{$r_t\ge 4\ln\lambda$}:
Then with probability $1-o(1)$, all offspring are invalid in prefix or suffix and therefore 
worse than the parent. The fitness function is $-n-1+\onemax$ in this case. Now, since the minimum number 
of bits flipped at strength~$2r_t$ is with probability $1-o(1)$ larger than the maximum number of bits flipped 
at strength $r_t/2$ (using Chernoff and union bounds), with probability $1-o(1)$ an offspring produced from strength
$r_t/2$ has best fitness and adjusts the strength. Hence, the probability of increasing the strength is at most $1/2-\epsilon$ again.

{$L\coloneqq (\ln \lambda)/4 \le r_t\le 4\ln \lambda \eqqcolon U$}: here we only know that the probability of decreasing the strength is 
at least~$1/4$ due to the random steps of the \saonelambdaea. However, a constant number of such decreasing 
steps is enough to reach strength at most~$L$ from the smallest possible strength above~$U$. Using a potential function with 
an exponential slope in the range $[L,U]$ like in \cite{DoerrGWYAlgorithmica19}, we arrive at a process that increases with probability at most $1/2-\epsilon$ 
and decreases with the remaining probability. We choose a constant~$\epsilon>0$ that is sufficiently small to cover all three cases.

We note that the probability of decreasing the strength is at least $1/2+\epsilon$ except for the case $r_t=2$, where the strength stays the same with probability 
at least $1/2+\epsilon$. Hence, for the process $X_t\coloneqq \log_2(r_t)$ that lives on the non-negative integers we obtain, writing 
$\Delta_t\coloneqq X_{t+1}-X_t$, that 
\[
\expect{e^{\eta \Delta_t}\mid \filtt; X_t>2} = e^{-\eta} \left(\frac{1}{2}+\epsilon\right) + e^{\eta} \left(\frac{1}{2}-\epsilon\right) \le 
1-2\eta\epsilon + \eta^2  \le \rho
\]
for a constant~$\rho<1$ if $\eta$ is chosen as a sufficiently small constant (depending on the constant~$\epsilon$). Similarly, given 
this choice of $\eta$, we immediately have 
\[
\expect{e^{\eta \Delta_t}\mid \filtt; X_t\le 2} \le D 
\]
for a constant $D>0$. If we choose $b$ in Theorem~\ref{theo:hajek-occ-times} as a sufficiently large constant, we obtain, noting  
$a=2$,  
\[
1-\epsilon - \frac{1-\epsilon}{1-\rho}De^{\eta(a-b)} \ge \frac{9}{10}
\]
Hence, the theorem states that in a phase of length $T$, the number of generations where $X_t>b$ holds, is at most 
$T/10$ with probability $1-2^{-\Omega(T)}$. Let $b^*=2^b$, \ie, the strength corresponding to $X_t=b$. We set 
$T\coloneqq (12/10)e^{b^*} mn/(b^*\lambda)$. Since a $\pre$-improving mutation has probability at least $(1-o(1))\lambda(b^*/n)e^{-b^*}$, we have an expected number 
of at least $(1-o(1))(27/25)m$ such mutations in the phase at with probability $1-2^{-\Omega(n)}$ at least~$m$ 
such mutations by Chernoff bounds. This is sufficient to reach the local optimum unless there are at least $(2/3)\xi\sqrt{n}$ 
$\suff$-improving mutations in the phase. Note that the choice of the constant~$\xi$ only impacts the length 
of the prefix 
 in lower-order terms that vanish in $O$-notation.

We bound the number of $\suff$-improving mutations separately for the points in time (\ie, generations) where $X_t\le b$ and where $X_t>b$.
For the first set of time points, we note that the probability of a $\suff$-improving mutation by~$k\ge 1$ is at most 
$(1+o(1)) \lambda (2/n^{3/2})^k e^{-2}$ since the term $x^2/e^{-x}$ takes its maximum at $x=2$. Using similar arguments based on Chernoff and union bounds 
as in the 
proof of Theorem~\ref{theo:plainooea-needhighmut}, we bound the total improvement of the $\suff$-value in at most $T \le (12/10)e^{b^*} n^2/(\lambda b^*)$ 
steps where $X_t\le b$ by $i_1\coloneqq (25/10) e^{b^*-2} \sqrt{n}/b^*$ with probability $1-2^{-\Omega(n^{1/9})}$. For the points of 
time where $X_t>b$ the probability of a $\pre$-improving mutation is maximized (up to lower-order terms) at strength~$b^*$ 
since the function $x^2/e^{-x}$ is monotonically decreasing for $x>2$. Assuming at most 
$(12/100)e^{b^*} n^2/b^*$ such time points (which assumption holds with probability at least $1-2^{-\Omega(n^2)}$, we obtain 
an expected number of $\suff$-improving mutations by~$1$ of at most 
\[
\frac{12}{100} e^{b^*} \frac{n^2}{b^*}  \frac{b^*}{n^{3/2}} e^{-b^*} =  \frac{12}{100} \sqrt{n}
\]
and using Chernoff and union bounds we bound the total improvement of the $\suff$-value in these generations  
by $i_2=(13/100)\sqrt{n}$ with $1-2^{-\Omega(n^{1/9})}$. Now, if we choose $\xi$ large enough, then
\[
i_1 + i_2 \le \frac{2}{3}\xi\sqrt{n}
\]
so that the prefix is optimized before the suffix with probability altogether $1-2^{-\Omega(n^{1/3})}$. 

Together with the analysis in Theorem~\ref{theo:sdooea-fail} for the case that the stagnation counter exceeds its threshold,
 this means that with probability $1-O(1/n)$ the local optimum is reached before the global one. Again arguing 
in the same way as in the proof of Theorem~\ref{theo:sdooea-fail}, the time to reach the global optimum 
from the local one is $2^{\Omega(n)}/\lambda$ with probability $1-2^{-\Omega(n)}$. The sum of all failure probabilities is $O(1/n)$.
\end{proof}

\section{Experiments}
\label{sec:experiments}
Our theoretical results are asymptotic. 
In this section, we show the results of the experiments\footnote{\url{https://github.com/DTUComputeTONIA/StagnationDetection}.} we did in order to see how the different algorithms perform in practice for small $n$.
	
	In the first experiment, we ran an implementation of Algorithms~\ref{alg:sdoneone} (\sdooea) and~\ref{alg:sdonelambda} (\saonelambdaea) on the \jump fitness function with jump size $m=4$ and $n$ varying from 40 to 160. We compared our algorithms against \ooea with standard mutation rate 1/n, \ooea with mutation probability $m/n$, and Algorithm \oofea from \cite{DoerrLMNGECCO17}
	with three different $\beta=\{1.5, 2, 4\}$.
	
\begin{figure} 
    \centering
	\includegraphics[width=\linewidth/2]{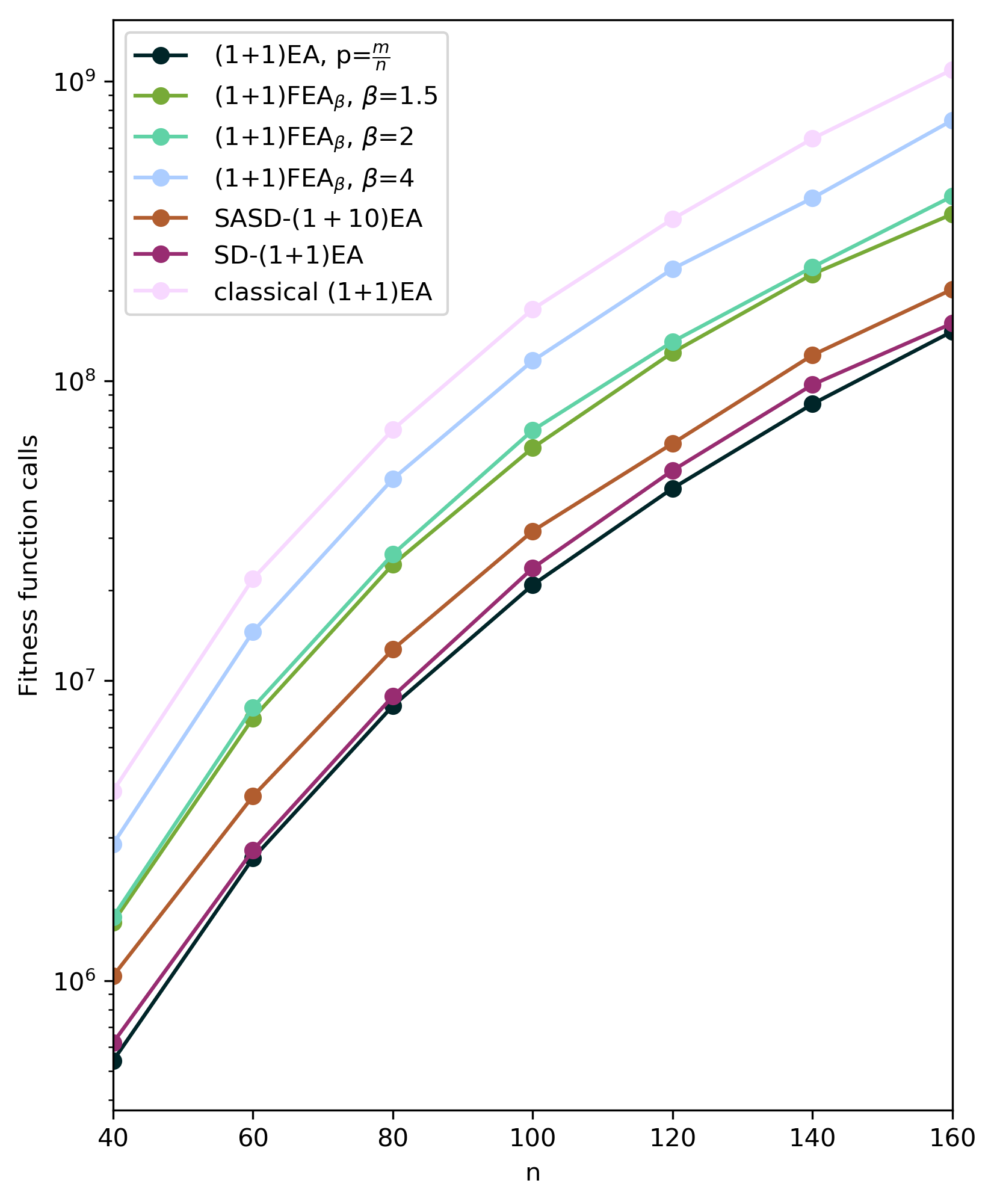}
	\caption{Average number of fitness calls (over 1000 runs) the mentioned algorithms take to optimize $\jump_4$.}\label{fig:exp_jump_line}
\end{figure}

	\begin{figure} 
		\includegraphics[width=\linewidth]{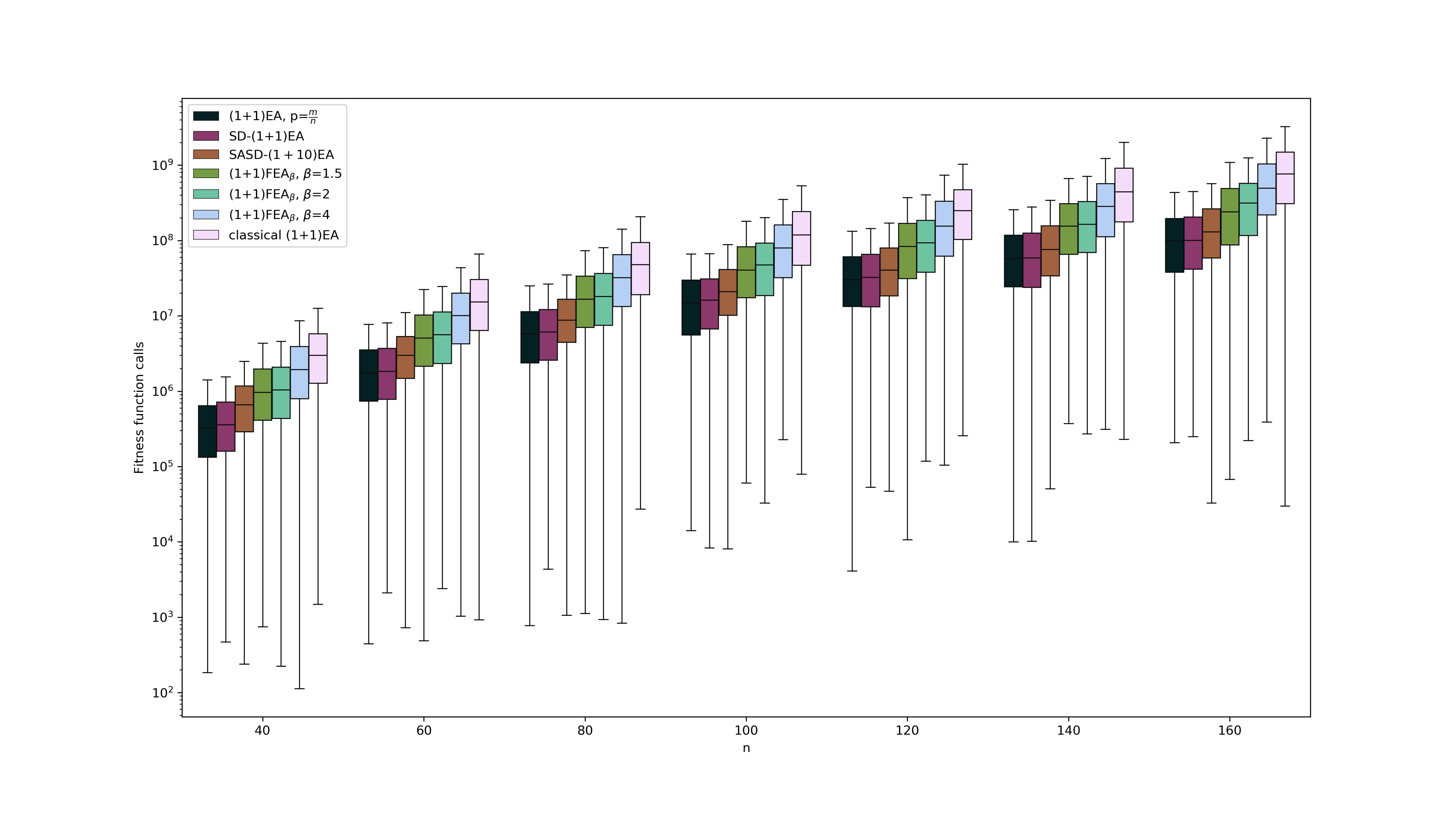}
		\caption{Box plots comparing number of fitness calls (over 1000 runs) the mentioned algorithms take to optimize $\jump_4$.} \label{fig:exp_jump_boxplot}
	\end{figure}

	In Figures~\ref{fig:exp_jump_line} and more precisely \ref{fig:exp_jump_boxplot}, we observe that stagnation detection technique makes the algorithm faster than the algorithms with heavy-tailed mutation operator \oofea. Also, Algorithm \sdooea is not much slower than the \ooea with mutation probability 
	$\frac mn$ even though it does not need the gap size.

\begin{table}[htb!]
\small
	\begin{tabular}{|l|l|l|l|l|l|l|}
\hline
\diagbox[width=\dimexpr \textwidth/8+2\tabcolsep\relax, height=1cm]{ n }{Algo.} & \begin{tabular}[c]{@{}l@{}}(1+1)EA \\ with p=$\frac 1n$\end{tabular}&\begin{tabular}[c]{@{}l@{}}(1+1)EA \\ with p=$\frac 2n$\end{tabular}& \begin{tabular}[c]{@{}l@{}}(1+1)EA \\ with p=$\frac 6n$\end{tabular} & \begin{tabular}[c]{@{}l@{}}(1+1)EA \\ with p=$\frac 8n$\end{tabular} & \begin{tabular}[c]{@{}l@{}}\small{SD-} \\ \small{($1+1$)EA}\end{tabular} &  \begin{tabular}[c]{@{}l@{}}\small{SASD-} \\ \small{($1+\ln n$)EA}\end{tabular} \\ \hline
200 &0.00000 &0.00000&0.01181&0.19380&0.00000&0.00000   \\ \hline
400 &0.00000 &0.00000&0.33858&0.87402&0.00100&0.00000   \\ \hline
600 &0.00000 &0.00000&0.42449&0.85950&0.00051&0.00000   \\ \hline
800 &0.00000 &0.00000&0.84000&0.97273&0.00056&0.00229   \\ \hline
1000&0.00000&0.00000&0.80769&0.97917&0.00058&0.00121   \\ \hline
	\end{tabular}
	
\caption{Ratio of successfully achieved global optimum where $\xi=3$ over 1000 runs.} \label{table:needhighmut3}
\end{table}
	In the second  experiment, we ran our algorithms and the classic \ooea with different mutation probabilities 
	on $\needhighmut_\xi$ with $n=\{200,400,600,800, 1000\}$ and $\xi=3$.
	
	The outcomes support that the theory from Section~\ref{sec:needhighmut} already holds for small~$n$. 
	In Table~\ref{table:needhighmut3}, one can see that for $\xi=3$, the  \ooea with $p=6/n$ and $8/n$ is much
	more successful to find global optimum points than the rest of the algorithms.


\section*{Conclusions}
We have designed and analyzed self-adjusting EAs for multimodal optimization. In particular, we have proposed a module 
called \emph{stagnation detection} that can be added to existing EAs without essentially changing their behavior 
on unimodal (sub)problems. Our stagnation detection keeps track of the number of unsuccessful steps and increases 
the mutation rate based on statistically significant waiting times without improvement. Hence, there is high evidence for 
being at a local optimum when the strength is increased.

Theoretical analyses  reveal that the \oneoneea equipped with stagnation detection optimizes the \jump function 
in asymptotically optimal time corresponding to the best static choice of the mutation rate. Moreover, we have proved 
a general upper bound for multimodal functions that can recover asymptotically runtimes on well-known example functions, 
and we have shown that on unimodal functions, the  \oneoneea with stagnation detection with high probability never deviates 
from the classical \oneoneea; also a related statement was proved for the self-adjusting \onelambdaea from \cite{DoerrGWYAlgorithmica19}. Finally, to show 
the limitations of the approach we have presented a function on which all of our investigated self-adjusting EAs provably fail to be efficient.

In the future, we would like to investigate our module for stagnation detection in other EAs and study its benefits 
on combinatorial optimization problems.

\section*{Acknowledgement}
	This work was supported  by a grant by the Danish Council for Independent Research  (DFF-FNU  8021-00260B).

\bibliographystyle{mynatbib_english}

\bibliography{references}

\begin{thebibliography}{37}
\expandafter\ifx\csname natexlab\endcsname\relax\def\natexlab#1{#1}\fi
\expandafter\ifx\csname url\endcsname\relax
  \def\url#1{\texttt{#1}}\fi
\expandafter\ifx\csname urlprefix\endcsname\relax\def\urlprefix{URL }\fi

\bibitem[{Antipov, Doerr and Karavaev(2019)Antipov, Doerr, and
  Karavaev}]{AntipovDKFOGA19}
Antipov, Denis, Doerr, Benjamin, and Karavaev, Vitalii (2019).
\newblock A tight runtime analysis for the {(1} + ({\(\lambda\)},
  {\(\lambda\)})) {GA} on {LeadingOnes}.
\newblock In \emph{Proc. of FOGA '19}, 169--182. ACM Press.

\bibitem[{Antipov, Doerr and Karavaev(2020)Antipov, Doerr, and
  Karavaev}]{AntipovDoerrKaravaev2020}
Antipov, Denis, Doerr, Benjamin, and Karavaev, Vitalii (2020).
\newblock The $(1 + (\lambda, \lambda))$ {GA} is even faster on multimodal
  problems.
\newblock \emph{CoRR}, \textbf{abs/2004.06702}.
\newblock \urlprefix\url{http://arxiv.org/abs/2004.06702}.

\bibitem[{Buzdalov, Doerr and Kever(2016)Buzdalov, Doerr, and
  Kever}]{BuzdalovDoerrKeverECJ16}
Buzdalov, Maxim, Doerr, Benjamin, and Kever, Mikhail (2016).
\newblock The unrestricted black-box complexity of jump functions.
\newblock \emph{Evolutionary Computation}, \textbf{24}(4), 719--744.

\bibitem[{Corus, Oliveto and Yazdani(2018)Corus, Oliveto, and
  Yazdani}]{CorusEtAlFastAISPPSN18}
Corus, Dogan, Oliveto, Pietro~Simone, and Yazdani, Donya (2018).
\newblock Fast artificial immune systems.
\newblock In \emph{Proc. of PPSN '18}, 67--78. Springer.

\bibitem[{Dang and Lehre(2016)}]{DangL16}
Dang, Duc-Cuong and Lehre, Per~Kristian (2016).
\newblock Self-adaptation of mutation rates in non-elitist populations.
\newblock In \emph{Proc. of PPSN '16}, 803--813. Springer.

\bibitem[{Doerr(2019)}]{DoerrGECCO19}
Doerr, Benjamin (2019).
\newblock A tight runtime analysis for the {cGA} on jump functions: {EDAs} can
  cross fitness valleys at no extra cost.
\newblock In \emph{Proc. of GECCO '19}, 1488--1496. {ACM} Press.

\bibitem[{Doerr and Doerr(2018)}]{DoerrDoerrAlgorithmica18}
Doerr, Benjamin and Doerr, Carola (2018).
\newblock Optimal static and self-adjusting parameter choices for the
  (1+({\(\lambda\)}, {\(\lambda\)})) genetic algorithm.
\newblock \emph{Algorithmica}, \textbf{80}(5), 1658--1709.

\bibitem[{Doerr and Doerr(2020)}]{DoerrDoerrParameterBookChapter}
Doerr, Benjamin and Doerr, Carola (2020).
\newblock Theory of parameter control for discrete black-box optimization:
  Provable performance gains through dynamic parameter choices.
\newblock In Doerr, B. and Neumann, F. (eds.), \emph{Theory of Evolutionary
  Computation -- Recent Developments in Discrete Optimization}, 271--321.
  Springer.

\bibitem[{Doerr, Doerr and K{\"{o}}tzing(2018)Doerr, Doerr, and
  K{\"{o}}tzing}]{DoerrDoerrKoetzingAlgorithmica18}
Doerr, Benjamin, Doerr, Carola, and K{\"{o}}tzing, Timo (2018).
\newblock Static and self-adjusting mutation strengths for multi-valued
  decision variables.
\newblock \emph{Algorithmica}, \textbf{80}(5), 1732--1768.

\bibitem[{Doerr, Fouz and Witt(2010)Doerr, Fouz, and
  Witt}]{DoerrFouzWittGECCO10}
Doerr, Benjamin, Fouz, Mahmoud, and Witt, Carsten (2010).
\newblock Quasirandom evolutionary algorithms.
\newblock In \emph{Proc.\ of GECCO~'10}, 1457--1464. ACM Press.

\bibitem[{Doerr et~al.(2019)Doerr, Gie{\ss}en, Witt, and
  Yang}]{DoerrGWYAlgorithmica19}
Doerr, Benjamin, Gie{\ss}en, Christian, Witt, Carsten, and Yang, Jing (2019).
\newblock The (1 + $\lambda$)~evolutionary algorithm with self-adjusting
  mutation rate.
\newblock \emph{Algorithmica}, \textbf{81}(2), 593--631.

\bibitem[{Doerr and Krejca(2018)}]{DoerrKrejcaGECCO18}
Doerr, Benjamin and Krejca, Martin~S. (2018).
\newblock Significance-based estimation-of-distribution algorithms.
\newblock In \emph{Proc. of {GECCO} '18}, 1483--1490. {ACM} Press.

\bibitem[{Doerr et~al.(2017)Doerr, Le, Makhmara, and Nguyen}]{DoerrLMNGECCO17}
Doerr, Benjamin, Le, Huu~Phuoc, Makhmara, R{\'{e}}gis, and Nguyen, Ta~Duy
  (2017).
\newblock Fast genetic algorithms.
\newblock In \emph{Proc. of {GECCO} '17}, 777--784. ACM Press.

\bibitem[{Doerr, Witt and Yang(2018)Doerr, Witt, and Yang}]{DoerrWY18}
Doerr, Benjamin, Witt, Carsten, and Yang, Jing (2018).
\newblock Runtime analysis for self-adaptive mutation rates.
\newblock In \emph{Proc. of GECCO '18}, 1475--1482. ACM Press.

\bibitem[{Doerr and Wagner(2018)}]{DoerrWagnerPPSN18}
Doerr, Carola and Wagner, Markus (2018).
\newblock Sensitivity of parameter control mechanisms with respect to their
  initialization.
\newblock In \emph{Proc. of {PPSN} '18}, 360--372. Springer.

\bibitem[{Doerr et~al.(2018)Doerr, Ye, van Rijn, Wang, and
  B\"{a}ck}]{DoerrEtAlTheoryGuidedBenchmark18}
Doerr, Carola, Ye, Furong, van Rijn, Sander, Wang, Hao, and B\"{a}ck, Thomas
  (2018).
\newblock Towards a theory-guided benchmarking suite for discrete black-box
  optimization heuristics: Profiling (1+\(\lambda\)) {EA} variants on {OneMax}
  and {LeadingOnes}.
\newblock In \emph{Proc. of GECCO '18}, 951–958. ACM Press.

\bibitem[{Droste, Jansen and Wegener(2002)Droste, Jansen, and
  Wegener}]{DrosteJW02}
Droste, Stefan, Jansen, Thomas, and Wegener, Ingo (2002).
\newblock On the analysis of the (1+1) evolutionary algorithm.
\newblock \emph{Theoretical Computer Science}, \textbf{276}, 51--81.

\bibitem[{Eiben, Marchiori and Valk{\'{o}}(2004)Eiben, Marchiori, and
  Valk{\'{o}}}]{EibenMVPPSN04}
Eiben, A.~E., Marchiori, Elena, and Valk{\'{o}}, V.~A. (2004).
\newblock Evolutionary algorithms with on-the-fly population size adjustment.
\newblock In \emph{Proc. of PPSN '04}, 41--50. Springer.

\bibitem[{Fajardo(2019)}]{FajardoGECCO19}
Fajardo, Mario A.~Hevia (2019).
\newblock An empirical evaluation of success-based parameter control mechanisms
  for evolutionary algorithms.
\newblock In \emph{Proc. of GECCO ’19}, 787--795. ACM Press.

\bibitem[{Friedrich, Quinzan and Wagner(2018)Friedrich, Quinzan, and
  Wagner}]{FriedrichQWGECCO18}
Friedrich, Tobias, Quinzan, Francesco, and Wagner, Markus (2018).
\newblock Escaping large deceptive basins of attraction with heavy-tailed
  mutation operators.
\newblock In \emph{Proc. of GECCO '18}, 293--300. ACM Press.

\bibitem[{Hajek(1982)}]{HajekDrift}
Hajek, Bruce (1982).
\newblock Hitting and occupation time bounds implied by drift analysis with
  applications.
\newblock \emph{Advances in Applied Probability}, \textbf{14}, 502--525.

\bibitem[{Hansen and Mladenovic(2018)}]{HansenMladenovic18}
Hansen, Pierre and Mladenovic, Nenad (2018).
\newblock Variable neighborhood search.
\newblock In Mart{\'{\i}}, Rafael, Pardalos, Panos~M., and Resende, Mauricio
  G.~C. (eds.), \emph{Handbook of Heuristics}, 759--787. Springer.

\bibitem[{Jansen, Jong and Wegener(2005)Jansen, Jong, and Wegener}]{JansenDW05}
Jansen, Thomas, Jong, Kenneth A.~De, and Wegener, Ingo (2005).
\newblock On the choice of the offspring population size in evolutionary
  algorithms.
\newblock \emph{Evolutionary Computation}, \textbf{13}, 413--440.

\bibitem[{Jansen and Wiegand(2004)}]{JansenWiegandECJ04}
Jansen, Thomas and Wiegand, R.~Paul (2004).
\newblock The cooperative coevolutionary {(1+1)} {EA}.
\newblock \emph{Evolutionary Computation}, \textbf{12}(4), 405--434.

\bibitem[{L{\"{a}}ssig and Sudholt(2011)}]{LassigSudholtFOGA11}
L{\"{a}}ssig, J{\"{o}}rg and Sudholt, Dirk (2011).
\newblock Adaptive population models for offspring populations and parallel
  evolutionary algorithms.
\newblock In \emph{Proc. of {FOGA} '11}, 181--192. {ACM} Press.

\bibitem[{Lengler(2018)}]{LenglerPPSN18}
Lengler, Johannes (2018).
\newblock A general dichotomy of evolutionary algorithms on monotone functions.
\newblock In \emph{Proc. of PPSN '18}, 3--15. Springer.

\bibitem[{Lissovoi, Oliveto and Warwicker(2020)Lissovoi, Oliveto, and
  Warwicker}]{LOWECJ19}
Lissovoi, Andrei, Oliveto, Pietro~S., and Warwicker, John~Alasdair (2020).
\newblock Simple hyper-heuristics control the neighbourhood size of randomised
  local search optimally for leadingones.
\newblock \emph{Evolutionary Computation}.
\newblock In print.

\bibitem[{Rodionova et~al.(2019)Rodionova, Antonov, Buzdalova, and
  Doerr}]{RodionovaABDGECCO19}
Rodionova, Anna, Antonov, Kirill, Buzdalova, Arina, and Doerr, Carola (2019).
\newblock Offspring population size matters when comparing evolutionary
  algorithms with self-adjusting mutation rates.
\newblock In \emph{Proc. of {GECCO} '19}, 855--863. ACM Press.

\bibitem[{Rohlfshagen, Lehre and Yao(2009)Rohlfshagen, Lehre, and
  Yao}]{RohlfshagenLehreYaoGECCO09}
Rohlfshagen, Philipp, Lehre, Per~Kristian, and Yao, Xin (2009).
\newblock Dynamic evolutionary optimisation: an analysis of frequency and
  magnitude of change.
\newblock In \emph{Proc. of {GECCO} '09}, 1713--1720. ACM Press.

\bibitem[{Rowe and Aishwaryaprajna(2019)}]{RoweAFOGA19}
Rowe, Jonathan~E. and Aishwaryaprajna (2019).
\newblock The benefits and limitations of voting mechanisms in evolutionary
  optimisation.
\newblock In \emph{Proc. of {FOGA} '19}, 34--42. {ACM} Press.

\bibitem[{Wegener(2001)}]{WegenerMethods}
Wegener, Ingo (2001).
\newblock Methods for the analysis of evolutionary algorithms on
  pseudo-{B}oolean functions.
\newblock In Sarker, Ruhul, Mohammadian, Masoud, and Yao, Xin (eds.),
  \emph{Evolutionary Optimization}. Kluwer Academic Publishers.

\bibitem[{Whitley et~al.(2018)Whitley, Varadarajan, Hirsch, and
  Mukhopadhyay}]{WhitleyVHM18}
Whitley, Darrell, Varadarajan, Swetha, Hirsch, Rachel, and Mukhopadhyay,
  Anirban (2018).
\newblock Exploration and exploitation without mutation: Solving the jump
  function in $\vartheta(n)$ time.
\newblock In \emph{Proc. of {PPSN} '18}, 55--66. Springer.

\bibitem[{Witt(2003)}]{WittCEC03}
Witt, Carsten (2003).
\newblock Population size vs. runtime of a simple {EA}.
\newblock In \emph{Proc.\ of the Congress on Evolutionary Computation
  (CEC~2003)}, vol.~3, 1996--2003. IEEE Press.

\bibitem[{Witt(2006)}]{WittECJ06}
Witt, Carsten (2006).
\newblock Runtime analysis of the ($\mu$+1)~{EA} on simple pseudo-boolean
  functions.
\newblock \emph{Evolutionary Computation}, \textbf{14}(1), 65--86.

\bibitem[{Witt(2008)}]{WittTCS08}
Witt, Carsten (2008).
\newblock Population size versus runtime of a simple evolutionary algorithm.
\newblock \emph{Theoretical Computer Science}, \textbf{403}(1), 104--120.

\bibitem[{Witt(2013)}]{WittCPC2013}
Witt, Carsten (2013).
\newblock Tight bounds on the optimization time of a randomized search
  heuristic on linear functions.
\newblock \emph{Combinatorics, Probability and Computing}, \textbf{22},
  294--318.

\bibitem[{Ye, Doerr and B{\"{a}}ck(2019)Ye, Doerr, and
  B{\"{a}}ck}]{YeDoerrBaeckCEC19}
Ye, Furong, Doerr, Carola, and B{\"{a}}ck, Thomas (2019).
\newblock Interpolating local and global search by controlling the variance of
  standard bit mutation.
\newblock In \emph{Proc. of CEC '19}, 2292--2299.

\end{thebibliography}
\end{document}